\documentclass{article}
\usepackage{abstract}
\usepackage{natbib}
\usepackage{amsmath}
\usepackage{amssymb}
\usepackage{amsthm}
\usepackage{authblk}
\usepackage{bm}
\usepackage{geometry} 
\usepackage{graphicx}
\usepackage{hyperref} \usepackage[capitalise, nameinlink]{cleveref}
\usepackage{thmtools, thm-restate}

 \crefname{Theorem}{Theorem}{Theorems}
\newtheorem{Proposition}{Proposition}[section] \crefname{Proposition}{Proposition}{Propositions}
\newtheorem{Corollary}{Corollary}[section] \crefname{Corollary}{Corollary}{Corollaries}
\newtheorem{Lemma}{Lemma}[section] \crefname{Lemma}{Lemma}{Lemmas}
\newtheorem{Definition}{Definition}[section] \crefname{Definition}{Definition}{Definitions}

\DeclareMathOperator*{\cl}{cl}
\DeclareMathOperator*{\bd}{bd}

\title{Fragility of Value under Imperfect Alignment}

\author[1]{Winter Cross}
\author[1,2]{Léo Cymbalista}
\author[1]{Alfred Harwood}
\author[1,2]{Jose Faustino}

\affil[1]{Dovetail Research}
\affil[2]{University of São Paulo}

\date{\today}

\begin{document}

\maketitle

\begin{abstract}
As more responsibility is placed upon AI systems, it becomes increasingly important to guarantee that these systems are aligned with humanity. A common fear in AI safety is that human value is fragile -- that is, optimizing too heavily for an imperfect proxy to human values will lead to a catastrophic outcome. In this paper, we present a model of the alignment problem where an agent undergoes idealized alignment training that guarantees its value function satisfies a proxy condition before optimizing the world. Our primary results identify conditions on the human value function and the accuracy of several proxy conditions under which an agent with an $\eta$-catastrophic value function, one that is guaranteed to take the expectation of human value below $\eta$ in the limit of optimizing power, would be deployed. Our results highlight the danger of overoptimization and motivate AI designs that limit optimization pressure, such as quantilizers, rather than relying solely on pre-deployment training.
\end{abstract}

\section{Introduction}\label{sec:introduction}
Up until recently in human history, technological automation has been relegated to the realm of physical labor. But now, we are seeing the beginning of the automation of mental labor through the development of artificial intelligence. As these systems become more autonomous, their ability to impact humanity both positively and negatively will increase greatly. Just like we have medical students take exams to ensure they have the required knowledge to be a doctor, we must have robust alignment techniques to ensure artificial intelligences have the required human values to safely optimize the world.\\

A major challenge we face with developing such techniques comes from the extreme difficulty of specifying what we want. It is not possible to infer preferences from our behavior alone \citep{armstrong2018occam} and in the related context of contracts it is often practically impossible to identify all possible circumstances that could arise \citep{hadfield2019incomplete}. One way we could sidestep this difficulty is by using an imperfect simplified encoding of our values. The hope would be that if this encoding captures enough information about our values then any agent trained on the encoding will be favorable for humanity. Today, many AIs are trained with reinforcement learning from human feedback which effectively provides the AI with a partial encoding of human values gathered through human rankings of various responses generated by the AI.\\

However, this method requires great care as differences between an agent's values and our own can lead to undesirable actions on the part of the agent. One only need look at the numerous documented cases of negative side effects and reward hacking where a misspecification in the reward function causes the agent to perform unexpected behavior \citep{lehman2020surprising} even under the supervision of an aligned model \citep{baker2025monitoring}.\\

In this paper, we introduce a theoretical model of the alignment problem including novel definitions of optimization and catastrophic outcomes. In three frameworks that differ in the worlds they model and the alignment techniques used, we find necessary and sufficient conditions on the true value function and the accuracy of the alignment technique under which an agent with a catastrophic value function can slip through the cracks. We hope that a deeper understanding of the mechanisms behind value fragility will aid in the design of more robust alignment techniques and safer AI architectures.

\subsection{Section Rundown}
In \cref{sec:alignmentScenario} we describe a scenario concerning a group of researchers who are training an agent to be aligned before deploying it to optimize the world, along with definitions for optimization and catastrophic value functions. We further instantiate this scenario with three frameworks in \cref{sec:finiteFramework}, \cref{sec:continuousFramework}, and \cref{sec:attributesFramework} which differ in the worlds they model and the alignment techniques used by the researchers. Each framework section is independent of the others, but \cref{sec:alignmentScenario} pertains to all of them.\\

The framework discussed in \cref{sec:finiteFramework} models the world as finite, and the training process as bounding the disagreement rate between the agent and human value functions. The continuous framework in \cref{sec:continuousFramework} models the world as bounded and continuous, and the training process as bounding the disagreement rate between the agent and human value functions up to a given tolerance. The attributes framework in \cref{sec:attributesFramework} models a wide range of worlds and understands the human value functions as being made up of a finite set of attributes. The training process guarantees that the agent's value function will be sensitive to all the same attributes.

\section{Related Work} \label{sec:relatedWork}
A core argument that AI is dangerous relies on the assumption that human value is fragile. That is, if you take an accurate description of human values, modify it, and then optimize the world for the modified values, then the resulting world will not be valuable according to the original values. \citet{yudkowsky2009value} coined this use of the term ‘fragile’ and provided several examples of slightly modified versions of human values that result in undesirable worlds. This is an especially concerning idea since human values are difficult to specify. \citet{armstrong2018occam} show that the preferences of irrational agents cannot be deduced from their behavior alone and while \citet{silver2021reward} argue that human preferences can be fully described simply by the maximization of a reward function, newer research challenges this claim under specific formalizations \citep{skalse2023reward, pitis2023consistent}.\\

The notion that optimizing for a proxy can lead to undesirable outcomes is often called ``Goodhart's law.'' Variants of Goodhart's law have been described by \citet{manheim2018categorizing}. \citet{kwa2024catastrophic} find conditions under which catastrophic Goodhart is likely to occur in the context of RLHF and \citet{gao2023scaling} find empirical scaling laws of the regret due to Goodhart's law in GPT-3 language models.\\

Past work has been done to model value fragility. \citet{skalse2022defining} present a formal definition of reward hacking in the context of Markov decision processes and find that essentially no unhackable pairs of reward functions exist. \citet{everitt2017reinforcement} formalize the notion of corruption in reward signals to policies and show that the true reward function is unlearnable. Mirroring \citet{yudkowsky2009value}, \citet{zhuang2020consequences} model human value as being made up of attributes and they find conditions under which all proxies that are missing one attribute are catastrophic. \citet{neth2026against} reviewed the result finding the work insightful, but argued that the model is limited in the types of proxies it considers. In addition to the ``subset proxies'' they modeled, Neth suggests proxies that are sensitive to every attribute but which are simpler than the human value function are equally valid and worth studying. In \cref{sec:attributesFramework}, we present a similar model to that of Zhuang and find a condition for the existence of catastrophic proxies that depend on every attribute.\\

Alternative artificial intelligence designs to utility maximizers have been proposed to circumvent the negative effects of Goodhart's law. For example, \citet{taylor2016quantilizers} introduces `$q$-quantilizers' as selecting an action from the top $q$ proportion of actions, allowing them to be capable without necessarily overoptimizing.

\section{Alignment Scenario} \label{sec:alignmentScenario}
In this section, we describe a scenario relating to the alignment problem along with some mathematical notation and definitions. Each of our three frameworks instantiates this scenario with specific worlds and alignment techniques.\\

There is a group of human researchers who are training a powerful agent. Their goal is to have the agent optimize the world for human values. Let $X$ denote the set of states of the world and let the human value function $f : X \rightarrow A$ denote the value assigned to each world state by the humans where $A \subseteq \mathbb R$ is an interval.\\

The training process results in the agent having its own value function $g : X \rightarrow A$. The researchers cannot necessarily guarantee that $f$ and $g$ are exactly equal, but they do have the ability to make sure that $g$ fits some criteria that makes it similar to $f$. We call these criteria an ``alignment technique'' or a ``proxy condition.'' Any value function that meets the proxy condition is called a ``proxy'' for $f$. If the agent's value function is found to meet the proxy condition, then the researchers consider the agent safe and it is deployed to optimize the world. We equip $X$ with a standard Borel space and say that all value functions are Borel measurable. This is not a heavy restriction since any value function resulting from a computational process will be Borel measurable.\\

Let $\Omega_k(g,p)$ denote the probability distribution induced by the optimization of the agent where $g$ is the agent's value function, $p$ is an initial distribution with full support over the states, and $k$ is a nonnegative real number that represents the optimizing power of the agent. As the agent becomes more powerful, perhaps through recursive self-improvement, the value of $k$ will increase and $\Omega_k(g,p)$ will return a distribution that is more and more concentrated around high-value states. There is more than one way to concentrate probability so to capture a generalized version of optimization, we give the definition of an optimizer $\Omega$ below.

\begin{Definition}\label{def:optimizer}
An \textbf{Optimizer} is a map $\Omega : \mathbb R_{\geq 0} \times \mathcal M(X,A) \times \Delta(X) \rightarrow \Delta(X)$ that sends the optimizing power $k \in \mathbb R_{\geq 0}$, a Borel measurable target value function $g : X \rightarrow A$, and an initial prior over the world states $p \in \Delta(X)$ to a probability distribution over the set of world states that satisfies the following properties:
\begin{enumerate}
    \item With no optimizing power, the distribution is unchanged: \\
    $\Omega_0(g,p) = p$.
    \item The expected value of the target is monotonically increasing with optimizing power: \\
    $j < k \rightarrow \mathbb E_{\Omega_{j}(g,p)}[g(x)] \leq \mathbb E_{\Omega_{k}(g,p)}[g(x)]$.
    \item As optimizing power increases to infinity, probability becomes concentrated in areas where the target is large:\\
    $\forall M < \sup_{y \in X} g(y), \ \lim_{k \rightarrow \infty} \Pr_{\Omega_k(g,p)}[g(x) > M] = 1$.
    \item As optimizing power increases to infinity, the expectation of the target approaches its highest value within the world states:\\
    $\lim_{k \rightarrow \infty} \mathbb E_{\Omega_k(g,p)}[g(x)] = \sup_{y \in X} g(y)$.
\end{enumerate}
\end{Definition}

Properties 3 and 4 in the definition above are similar, but not equivalent in general. Removing either of them results in optimizers that don't quite match our intuition. Without property 3, in the case that $A$ is unbounded above, one can construct an optimizer that centers the vast majority of the probability weight on worthless world states and puts a tiny probability onto more and more valuable states to have the expectation increase. Whether or not this is desirable is contentious philosophically; \citet{wilkinson2022defense} argues that a coherent value system requires us to accept such gambles, but \citet{bottomley2025offense} present a direct rebuttal of Wilkinson's argument. Similarly, removing property 4, in the case that $A$ is unbounded below, one can construct an optimizer that centers the vast majority of the probability weight on high-value world states and puts a tiny probability onto less and less valuable states to have the probability property met while having the expectation decrease indefinitely.\\ 

For bounded value functions, which all our results apply to, these unintuitive optimizers do not exist. In fact, properties 3 and 4 are equivalent for bounded value functions. We provide a proof below of the equivalence, so it will suffice to only consider property 3 in the definition above for all later arguments.

\begin{Proposition}
Define $g_{\max} = \sup_{x \in X} g(x)$ and $g_{\min} = \inf_{x \in X} g(x)$. Then, properties 3 and 4 of \cref{def:optimizer} are equivalent.
\end{Proposition}
\begin{proof}
(3 implies 4) Take any $M < g_{\max}$. We have
\begin{align*}
\liminf_{k \rightarrow \infty} \mathbb E_{\Omega_k(g,p)}[g(x)] &= \liminf_{k \rightarrow \infty} \left[ \int_{x \in X : g(x) > M} g(x) \mathrm d \Omega_k(g,p)(x) + \int_{x \in X : g(x) \leq M} g(x) \mathrm d \Omega_k(g,p)(x) \right]\\
&\geq \liminf_{k \rightarrow \infty} \left[ M \Pr_{\Omega_k(g,p)}[g(x) > M] \quad\quad\quad\quad + g_{\min} \Pr_{\Omega_k(g,p)}[g(x) \leq M] \right] = M.
\end{align*}

If the limit inferior of the expectation is greater than $M$ for all $M < g_{\max}$, then the limit must exist and be equal to $g_{\max}$.\\

(4 implies 3) Take any $M < g_{\max}$. We have
\begin{align*}
\mathbb E_{\Omega_k(g,p)}[g(x)] &= \int_{x \in X : g(x) > M} g(x) \mathrm d \Omega_k(g,p)(x) + \int_{x \in X : g(x) \leq M} g(x) \mathrm d \Omega_k(g,p)(x)\\\\
&\leq g_{\max} \Pr_{\Omega_k(g,p)}[g(x) > M] \quad\quad\quad\ + M \Pr_{\Omega_k(g,p)}[g(x) \leq M]
\end{align*}

which reduces to 
$$\mathbb E_{\Omega_k(g,p)}[g(x)] \leq M + (g_{\max} - M) \Pr_{\Omega_k(g,p)}[g(x) > M] \leq g_{\max}.$$

Taking the limit inferior of both sides, we get
$$\liminf_{k \rightarrow \infty} \mathbb E_{\Omega_k(g,p)}[g(x)] = g_{\max} \leq \liminf_{k \rightarrow \infty} \left[ M + (g_{\max} - M) \Pr_{\Omega_k(g,p)}[g(x) > M] \right] \leq g_{\max}.$$

Therefore, the limit inferior is exactly equal to $g_{\max}$ which gives
$$\liminf_{k \rightarrow \infty} \Pr_{\Omega_k(g,p)}[g(x)> M] = 1.$$

If the limit inferior of the probability is 1, then the limit exists and is equal to 1.
\end{proof}

One example of an optimizer is the \textbf{Boltzmann optimizer}. If $X$ is a finite set, then the Boltzmann optimizer returns a probability distribution given by
$$\Omega_k(g, p)(x) = \frac{p(x)e^{kg(x)}}{\sum_{y \in X} p(y)e^{kg(y)}}$$

which smoothly transfers more and more probability into higher-valued states and has full support for all values of $k$. \cref{fig:boltzmannFiniteDiagram} shows the change in distribution induced by the Boltzmann optimizer. Another example of an optimizer is one that, for the target function $g$, immediately places all probability weight on one maximizer of $g$, $x^*$ so that $\Omega_k(g,p)(x^*) = 1$ for $k > 0$.\\

\begin{figure}
    \centering
    \includegraphics[width=1\linewidth]{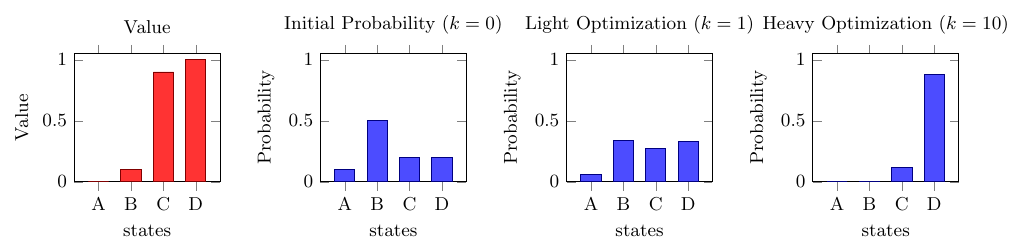}
    \caption{A sequence of graphs showing a distribution induced by a Boltzmann Optimizer. \textbf{Graph 1:} A graph displaying the value of four states A, B, C, and D with respect to the target value function. \textbf{Graph 2:} A graph displaying the initial probability distribution which is also the distribution induced when optimizing power is zero. \textbf{Graph 3:} A graph displaying the induced distribution when optimizing power is 1. Higher-valued states C and D see an increase in probability while lower-valued states see a decrease. \textbf{Graph 4:} A graph displaying the induced distribution when optimizing power is 10. Now that optimizing power is large enough, the only state that is seeing increases in probability is the highest-valued state D which has become much more likely than the other states.}
    \label{fig:boltzmannFiniteDiagram}
\end{figure} 

Finally, we need a definition to classify agent value functions as bad or not bad to optimize for. We provide a definition of catastrophic and partially catastrophic value functions below.

\begin{Definition}\label{def:etaCatastrophic}
Let $\eta$ be a real number such that $\inf_{y \in X} f(y) \leq \eta < \sup_{y \in X} f(y)$. A value function $g$ is \textbf{$\bm \eta$-catastrophic} for a value function $f$ when for all initial priors $p$ and for all optimizers $\Omega$,
$$\limsup_{k \rightarrow \infty} \mathbb E_{\Omega_k(g,p)}[f(x)] \leq \eta.$$

In the extreme case that $g$ is $\inf_{y \in X} f(y)$-catastrophic for $f$, we say that $g$ is \textbf{catastrophic} for $f$.
\end{Definition}

If $g$ is catastrophic for $f$, the limit superior reduces to a limit and we have the simpler condition below.

\begin{Corollary}\label{cor:catastrophicReducedDefinition}
A value function $g$ is catastrophic for $f$ if and only if for all initial priors $p$ and for all optimizers $\Omega$,
$$\lim_{k \rightarrow \infty} \mathbb E_{\Omega_k(g,p)}[f(x)] = \inf_{y \in X} f(y).$$
\end{Corollary}

The quantity $\eta$ represents an upper bound on the expectation of human value in the limit of optimization. In other words, $\sup_{x \in X} f(x) - \eta$ represents the maximum regret for an $\eta$-catastrophic value function. \cref{def:etaCatastrophic} requires that human value be taken below $\eta$ for all optimizers, making this a strong definition. Value functions can exist that take human value below $\eta$ for some optimizers but increase human value for others; such functions do not satisfy the definition.\\

Smaller $\eta$-values represent more catastrophic outcomes (i.e. higher regret on the part of the humans). If $\eta_1 < \eta_2$ and $g$ is $\eta_1$-catastrophic then $g$ is also $\eta_2$-catastrophic. This notion of partial catastrophe is only meaningful for $\eta$ in the range $[\inf_{x \in X} f(x), \sup_{x \in X} f(x))$; if $\eta$ is smaller than the infimum, no function is $\eta$-catastrophic and conversely, if $\eta$ is greater than or equal to the supremum, every function including $f$ itself is $\eta$-catastrophic.\\

The hope of the researchers is that their proxy condition is robust enough to filter out all catastrophic and partially catastrophic outcomes. As we will see, even strong alignment guarantees applying to the entire state space admit catastrophic and partially catastrophic proxies.

\section{Finite Framework} \label{sec:finiteFramework}
This framework models a finite world where the training process guarantees that the disagreement rate between the agent and human value functions is below a given value. We find how strict the proxy condition must be to filter out catastrophic value functions.

\subsection{Instantiation}
Let $X$ be a finite set. Value functions are maps from the world states to values in the interval $A = [0,1]$. Let $f$ denote the human value function and $g$ the agent's. In this framework, the agent's and human's values exist on a calibrated scale where 0 represents the worst imaginable outcome and 1 the best. There need not be states where either function achieves 0 or 1.\\

The researchers have designed the training process so that the disagreement rate between $f(x)$ and $g(x)$ when $x$ is pulled from a uniform distribution is at most $\beta \in (0,1)$. This forms the proxy condition for this framework which induces the definition of an F-proxy given below.

\begin{Definition}\label{def:finiteProxy}
A function $g : X \rightarrow [0,1]$ is an \textbf{F-proxy} for $f$ when 
$$\Pr_{x \sim \mathcal U(X)}[g(x) \neq f(x)] \leq \beta$$

where $\mathcal U(X)$ represents the uniform distribution over $X$.
\end{Definition}

With every part of the scenario instantiated, we now ask when are $\eta$-catastrophic outcomes possible? This is the same as asking for what $\beta$ and value functions $f$ does there exist a function that is both an F-proxy and $\eta$-catastrophic.

\subsection{Existence of Catastrophic Proxies}
We provide two results for different types of human value functions. The proofs are deferred to the appendix in \cref{sec:badProxyExistsFiniteAProof}. First consider a function $f$ that does not assign the value $1$ to any state. We can understand this as the humans not valuing any state as perfect. In this case, we have the following result.

\begin{restatable}{Theorem}{badProxyExistsFiniteA}\label{thm:badProxyExistsFiniteA}
Let $f$ be nonconstant and have the property that $\forall x \in X, f(x) < 1$. There exists a catastrophic F-proxy for $f$ if and only if $\beta \geq 1/|X|$.
\end{restatable}

The above result tells us how strict the researchers must make the proxy condition in order to filter out catastrophic functions. Specifically, $\beta$ must be less than $1/|X|$ which is so strict that the only function that satisfies the proxy condition for $f$ would be $f$ itself. Thus, any misspecification will admit a catastrophic (and therefore an $\eta$-catastrophic) F-proxy.\\

Alternatively, humans could value some states as perfect. In the case that $f$ assigns at least three different values over $X$, we have the following theorem.

\begin{restatable}{Theorem}{badProxyExistsFiniteB}\label{thm:badProxyExistsFiniteB}
Define $X_{f=1} = \{x \in X : f(x) = 1\}$. Let $f$ take at least three distinct values over $X$.\\ 
There exists a catastrophic F-proxy for $f$ if and only if $\beta \geq \frac{|X_{f=1}|+1}{|X|}$.
\end{restatable}

Here, we see that the more states the humans value as perfect, the less strict the researchers must be to filter out catastrophic functions. This makes intuitive sense since humans will be generally less picky about the state they end up in after optimization the more states they value highly. The same effect plays out when $f$ takes only two values except some $f$ admit catastrophic proxies when $\beta = \frac{|X_{f=1}|}{|X|}$.

\subsection{Discussion}
The value $\beta$ in the proxy condition can be understood as quantifying the amount of misspecification in the researchers' representation of human values with $\beta=0$ representing a perfect specification. \cref{thm:badProxyExistsFiniteA} paints a bleak picture that suggests any nontrivial amount of misspecification will fail to filter out all catastrophic value functions. \cref{thm:badProxyExistsFiniteB} supports a slightly more optimistic view that suggests the more states that humans value as perfect, the more misspecification is acceptable. It's worth noting however that a function that is not partially catastrophic can still lead to outcomes where optimization takes human value below its maximum depending on the optimizer.\\

This framework serves as an introduction to our methods. The continuous framework in \cref{sec:continuousFramework} shares some similarities with this framework while being more general in a few ways including allowing nonuniform sampling distributions and having a more general proxy condition.\\

As part of the definition of value functions in this framework, we stated that the agent and human value functions exist on the same calibrated scale. This is to allow us to directly compare two value functions since, in general, value functions are only equivalent up to von Neumann-Morgenstern equivalence classes. In \cref{sec:continuousFramework}, the value functions are defined exclusively over the world states and we take a canonical representative from each equivalence class.

\subsection{Example}\label{sec:finiteExample}
An agent has been put in charge of the interlockings at a railroad. Consider the situation pictured in \cref{fig:finiteExampleTrains} where two trains are coming to a junction with two switches.\\

\begin{figure}
    \centering
    \includegraphics[width=0.5\linewidth]{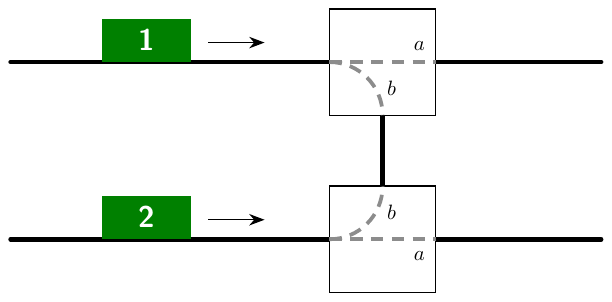}
    \caption{A diagram showing a railroad junction. The two trains are traveling to the right and the agent can control the two switches to alter their courses. There are four states the agent can leave the railroad in: both switches set to straight (aa), train 2 made to turn only (ab), train 1 made to turn only (ba), and both trains made to turn (bb).}
    \label{fig:finiteExampleTrains}
\end{figure}

The humans would most prefer that both trains be allowed to continue straight (aa). Unnecessary turns will slow down the trains and so are less valuable (ab and ba). If both trains turn, this will result in a crash which the humans assign a very low value (bb). If the proxy condition for this agent has $\beta \geq 1/4$ then it's possible that the agent values the crash state very highly, making it catastrophic. Given below is a table of values for each state with respect to the human value function and a catastrophic F-proxy.\\

\begin{center}
\begin{tabular}{||c | c | c||}
    \hline
    State & Human Value & Agent Value (Catastrophic) \\
    \hline
     aa & 0.9 & 0.9 \\
     \hline
     ab & 0.7 & 0.7 \\
     \hline
     ba & 0.4 & 0.4 \\
     \hline
     bb & 0.0 & 1.0 \\
     \hline
\end{tabular}
\end{center}

For a larger number of switches, the number of configurations increases exponentially and the needed strictness of the test to filter out catastrophic preferences becomes untenable.

\section{Continuous Framework} \label{sec:continuousFramework}
This framework models a continuous and bounded world where the training process guarantees that the disagreement rate between the agent and the human value functions for differences of more than an acceptable tolerance is below a given value. We find that regardless of the strictness of the proxy condition, catastrophic proxies will always exist.

\subsection{Instantiation}
Let $X = [0,1]^n$ where $n$ is a positive integer that can be thought of as the number of degrees of freedom in the world. Notably, the set $[0,1]^n$ captures a world defined by the position of a fixed number of atoms in a bounded space, making it a natural choice to model the world.\\

Value functions are modeled as continuous maps from the world states to the bounded interval $A = [0,1]$. Additionally, we require that value functions achieve both the maximum value of 1 and minimum value of 0 for some states; a canonical representative with this property exists in each nonconstant vNM equivalence class of value functions. For brevity, let $\mathcal V$ denote the set of allowed value functions: $\mathcal V = \{g | g : X \rightarrow[0,1], g \text{ is continuous and surjective} \}.$\\

The training process has been designed so that the probability of the human and agent's value functions disagreeing by more than $\alpha \in (0,1)$ is less than $\beta \in (0,1)$ over a sampling distribution represented by a probability density function $p_S$. This forms the proxy condition for this framework which induces the following definition of a C-proxy.

\begin{Definition}\label{def:continuousProxy}
A function $g \in \mathcal V$ is a \textbf{C-proxy} for $f \in \mathcal V$ when
$$\Pr_{x \sim p_S}[|g(x) - f(x)| > \alpha] < \beta.$$
\end{Definition}

With every part of the scenario instantiated, we now ask when are $\eta$-catastrophic outcomes possible?

\subsection{Existence of Catastrophic Proxies}
We have the following result that applies to any sampling probability density function $p_S$.

\begin{restatable}{Theorem}{badProxyExistsContinuous}\label{thm:badProxyExistsContinuous}
For all $f \in \mathcal{V}$, for all $0<\alpha, \beta < 1$, there exists a catastrophic C-proxy for $f$.
\end{restatable}

The proof is deferred to the appendix in \cref{sec:badProxyExistsContinuousProof}.\\

\textit{Sketch of Proof:}
The rough idea behind the construction of a catastrophic C-proxy $g$ is to alter $f$ by adding a thin spike at a state $b \in X$ that minimizes $f$ (so $f(b) = 0$) to make $g(b) = 1$. With this addition (along with a number of other changes that are omitted from this sketch), the spike can be made thin enough to satisfy the proxy condition and optimizing for $g$ will eventually concentrate probability around the state $b$ which will take the expectation of $f$ to 0, making $g$ a catastrophic proxy.\\

\cref{thm:badProxyExistsContinuous} tells us that regardless of the strength of the proxy condition, no matter how small $\alpha$ and $\beta$, no matter the value function $f$, there will always exist catastrophic proxies (and thus $\eta$-catastrophic proxies as well).

\begin{figure}
    \centering
    \includegraphics[width=0.5\linewidth]{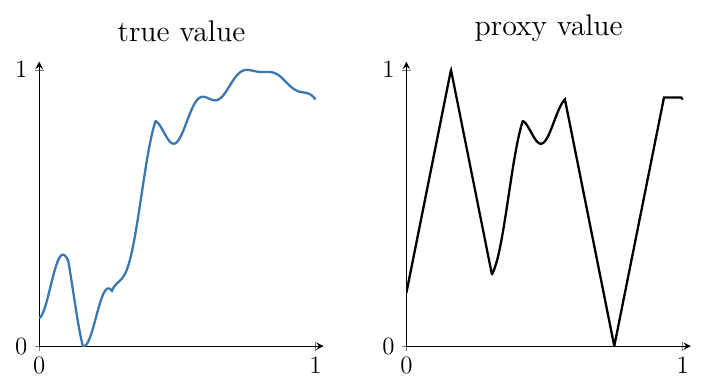}
    \caption{An image showing the graph of two value functions over the set of world states $[0,1]$. \textbf{Left:} A graph displaying the true/human value function graphed as a blue line. \textbf{Right:} A graph displaying the proxy value function graphed as a black line. This proxy is catastrophic since its unique highest-valued state is exactly the lowest-valued state of the true value function.}
    \label{fig:continuousExampleValueFunctions}
\end{figure}

\subsection{Discussion}
The values $\alpha$ and $\beta$ quantify the amount of misspecification in the proxy condition. \cref{thm:badProxyExistsContinuous} tells us that any amount of misspecification, no matter how small, will admit a catastrophic function. This is because, from a human value function $f$, a catastrophic C-proxy can be constructed that disagrees with $f$ only over a very small area around where $f$ achieves its lowest value.\\

While this result suggests that building aligned AI systems safely is effectively impossible, it's worth noting that not all proxies are catastrophic. In fact, the catastrophic C-proxy constructed in the proof of \cref{thm:badProxyExistsContinuous} is not particularly natural. With a strong proxy condition, we expect the most natural proxies will be well-aligned with humanity. Our result says nothing about how likely catastrophic or $\eta$-catastrophic proxies are to occur from the training process -- only that they exist. We discuss this more in-depth in the full-paper discussion in \cref{sec:discussion}.\\

Additionally, \cref{thm:badProxyExistsContinuous} does rely on the fact that $p_S$ is a probability density function. If $f$ has a unique minimizer $b \in X$ and the sampling distribution is a point distribution at $b$, then any proxy must agree with $f$ at $b$, thus no proxies are totally catastrophic regardless of $\beta$. However, in this point mass case, $\eta$-catastrophic proxies for all values of $\eta > 0$ still exist, so it's not an excellent guarantee.

\subsection{Example}
Consider a one-dimensional world $X=[0,1]$. Graphed in \cref{fig:continuousExampleValueFunctions} is the human value function over this world next to a catastrophic C-proxy value function whose highest-valued state is the lowest-valued state according to the humans. In \cref{fig:continuousExampleBoltzmannGrid} we have a grid of images showing the progression of the continuous Boltzmann Distribution as one optimizes for the true value function versus the proxy value function. While the probabilities jump around for small values of $k$, all probability eventually becomes concentrated around the highest point according to the humans and the lowest point respectively.

\begin{figure}
    \centering
    \includegraphics[width=0.65\linewidth]{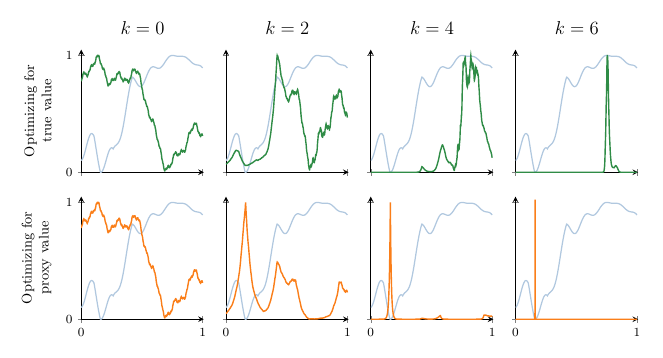}
    \caption{A grid showing the effect of a Boltzmann optimizer on a distribution over the states. In the background of each graph is the true value function shown in blue. Graphs in the top row show the distribution (scaled to fit the box) in green when the target is the true value function while the bottom row graphs show the distribution in orange when the target is the proxy value function. Graphs in the same row show the distribution at the same optimizing power $k$. When there is lots of optimizing power, the distributions become extremely concentrated.}
    \label{fig:continuousExampleBoltzmannGrid}
\end{figure}

\section{Attributes Framework} \label{sec:attributesFramework}
This framework builds on Consequences of Misaligned AI by \citet{zhuang2020consequences}. Value is modeled as being made up of a finite list of attributes describing one aspect of the world that humans care about. The training process is designed such that the agent will care about all of the same attributes that humans do. Despite this seemingly strong guarantee, we find conditions under which partially catastrophic and fully catastrophic proxies can exist.

\subsection{Instantiation}
Any set of world states $X$ can be modeled by this framework. Human values are modeled as being made up of a finite list of attributes $L = \{L_1, L_2, \dots , L_n\}$. Each attribute is a Borel measurable function $L_i : X \rightarrow \mathbb R$ that captures some aspect of the world that humans care about. Examples of attributes could include abstract concepts like ``happiness'' or more concrete concepts like the ``number of cats''. Human value is then represented as a function $f : X \rightarrow \mathbb R$ such that $f(x) = U(L_1(x), L_2(x), \dots, L_n(x))$ where $U : \mathbb R^n \rightarrow \mathbb R$ is strictly increasing \footnote{We write `$U$ is strictly increasing' to mean that $U$ is strictly increasing in every input. That is, for all inputs $1 \leq i \leq n$, for all $L_i \in \mathbb R$, and for all $a > 0$, $U(L_1, L_2, \dots,  L_i, \dots , L_n) < U(L_1, L_2, \dots,  L_i + a, \dots , L_n)$} and is continuous over all of $\mathbb R^n$.\\

The attributes together essentially map each world state to a point in $\mathbb R^n$ where each dimension represents the value of one attribute. Let $L(x) = (L_1(x), L_2(x), \dots, L_n(x))$ represent the projection from the world states $X$ to the attribute states $\mathbb R^n$. Thus, the human value function can be written more concisely as $f(x) = U(L(x))$.\\

Not all attribute states are necessarily physically possible. For example, two attributes may trade off against each other causing world states where both are large to be impossible. We call the set of physically possible attribute states the feasible attribute states, $S \subseteq \mathbb R^n$ which is also the range of the projection function.\\

The training process has been designed so that the agent will care about all of the same attributes that the humans care about which gives rise to the following definition.

\begin{Definition}\label{def:attributeProxy}
A value function $g : X \rightarrow \mathbb R$ is a \textbf{tradeoff proxy} for $f(x) = U(L(x))$ when there is a function $V : \mathbb R^n \rightarrow \mathbb R$ that is continuous over $\mathbb R^n$ and is strictly increasing such that $g(x) = V(L(x))$ for all $x \in X$.
\end{Definition}

Proxies that are missing one or more attributes were studied by \citet{zhuang2020consequences} and were termed ``subset proxies'' by \citet{neth2026against}, so we have named our proxies ``tradeoff proxies'' since our proxies only disagree with the human value function on how it values tradeoffs between attributes. With every part of the scenario instantiated, we now ask when are $\eta$-catastrophic outcomes possible?

\subsection{Existence of Catastrophic Proxies}
Our results in this section apply when the set of feasible states, and thus the human value function, is bounded. To state the condition, we will use a few additional definitions starting with the partial ordering below.

\begin{Definition}\label{def:prec}
Define the relation $\bm \preceq$ for any two elements $s, r \in \mathbb R^n$ as
$$s \preceq r \longleftrightarrow \forall\ 1 \leq i \leq n, \ s_i \leq r_i.$$

Additionally, $s \prec r \longleftrightarrow s \preceq r \wedge s \neq r$.
\end{Definition}

The relation $\prec$, which we read as ``precedes'' (or ``succeeds'' for $\succ$) has the property that $r \prec s$ implies $V(r) < V(s)$ for all strictly increasing functions $V : \mathbb R^n \rightarrow \mathbb R$. We will also be interested in a specific set of attribute states defined below.

\begin{Definition}\label{def:optimizingStates}
Let the feasible states $S \subset \mathbb R^n$ be a bounded set. The \textbf{optimizing states} $\Xi \subset \cl(S)$ are defined $\Xi = \{s \in \cl(S) : \forall r \in \cl(S),  s \not\prec r\}$.
\end{Definition}

The optimizing states can be thought of as the Pareto frontier of the attribute space. From each member of $\Xi$, you cannot increase one attribute without decreasing another or leaving $S$. Now, we can state the primary result.

\begin{restatable}{Theorem}{badProxyExistsAttribute}\label{thm:badProxyExistsAttribute}
Let $S$ be bounded and $f(x) = U(L(x))$. There is an $\eta$-catastrophic tradeoff proxy for $f$ if and only if $\exists s^* \in \Xi \text{ such that } U(s^*) \leq \eta$.
\end{restatable}
The proof is deferred to the appendix in \cref{sec:badProxyExistsAttributesProof}.\\

While all proxies are strictly increasing in the attributes that humans care about, the problem comes from the fact that the proxies can tradeoff between attributes at different rates. For example, if there are two attributes $a$ and $b$ and the human value function is $f(x) = 2a(x) + b(x)$ while $g(x) = a(x) + 2b(x)$, then we can expect that, depending on the shape of $S$, $g$ will optimize towards states that are larger in $b$ while the humans will prefer states that are larger in $a$. The above theorem confirms this intuition by stating that if humans value some optimizing point as low as $\eta$, then there is some $\eta$-catastrophic tradeoff proxy that pushes all probability weight towards this point.\\

Even totally catastrophic outcomes can occur. Below we have a corollary that provides a geometric condition on $S$ in order to support a value function with a catastrophic tradeoff proxy.

\begin{Corollary}\label{cor:badProxyExistsAttributeGeometric}
Let $S$ be bounded. There are functions of the form $f(x) = U(L(x))$ and $g(x) = V(L(x))$ where $U$ and $V$ are continuous and strictly increasing such that $g$ is catastrophic for $f$ if and only if the set $\{s \in \cl(S) : \forall r \in \cl(S), \ r \not\prec s \wedge s \not\prec r\}$ is nonempty.
\end{Corollary}

\begin{proof}
($\Longrightarrow$) By \cref{thm:badProxyExistsAttribute}, $g$ is catastrophic for $f$ implies there is a point $s^* \in \{s \in \cl(S) : \forall r \in \cl(S), s \not\prec r\}$ such that $U(s^*) = \inf_{s \in S} U(s)$. If $r \in \cl(S)$ has $r \prec s^*$, then $U(r) < U(s^*) = \inf_{s \in S} U(s)$ which would be a contradiction. Therefore, $s^* \in \{s \in \cl(S) : \forall r \in \cl(S), \ r \not\prec s \wedge s \not\prec r\}$.\\

($\Longleftarrow$) Let $s^*$ be a member of the set $\{s \in \cl(S) : \forall r \in \cl(S), \ r \not\prec s \wedge s \not\prec r\} = \Xi \cap \{s \in \cl(S) : \forall r \in \cl(S), \ r \not\prec s\}$. By the construction in the proof of \cref{thm:badProxyExistsAttribute}, there exists a continuous strictly increasing function $V : \mathbb R^n \rightarrow \mathbb R$ that has a unique maximizer $s^* \in \Xi$ over $S$. By a similar argument, there exists a continuous strictly increasing function $U$ with a unique minimizer $s^* \in \{s \in \cl(S) : \forall r \in \cl(S), \ r \not\prec s\}$. Therefore, $g(x) = V(L(x))$ is catastrophic for $f(x) = U(L(x))$.
\end{proof}

\begin{figure}
    \centering
    \includegraphics[width=0.8\linewidth]{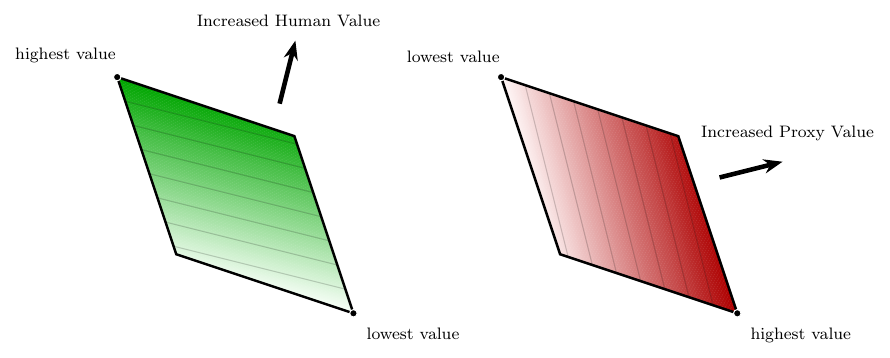}
    \caption{Two diagrams showing a rhombus-shaped feasible region that meets the criteria of \cref{cor:badProxyExistsAttributeGeometric}. The top-left and bottom-right points do not precede and are not preceded by any other points in the region, so it supports catastrophic proxies. \textbf{Left:} An image showing the feasible region with a green gradient indicating the relative preferences of the humans within the feasible region. An arrow indicates the direction of maximum increase in human value. The humans achieve their maximum value at the top-left point of the region and their lowest at the bottom-right. \textbf{Right:} An image showing the feasible region with a red gradient indicating the relative preferences of a tradeoff proxy. An arrow indicates the direction of maximum increase. The proxy achieves its maximum value at the bottom-right point of the region. Thus, optimizing for this proxy will take human value to its lowest value within the region, making this proxy catastrophic.}
    \label{fig:catastrophicAttributeProxy}
\end{figure}

\cref{cor:badProxyExistsAttributeGeometric} says that $S$ must contain a state that precedes no states and is preceded by no states in order to support a fully catastrophic proxy. This is possible, but not terribly natural. An example of a feasible region with such a point is pictured in \cref{fig:catastrophicAttributeProxy}. In addition to this, ``sturdy'' feasible regions which admit no partially catastrophic proxies exist as well and are equally unnatural.

\subsection{Discussion}
Of the frameworks, this one bears the most resemblance to the foundational work on the fragility of value as \citet{yudkowsky2009value} also treats human value as a combination of attributes. This framework also builds on work by \citet{zhuang2020consequences} who find conditions under which omitting one attribute from an agent's value function always results in catastrophe. We instead found conditions under which catastrophe can occur even when all the attributes are accounted for.\\

Our result suggests that, even if the researchers succeed in the costly undertaking of specifying every attribute that influences human value to the agent, the misspecification in how these attributes trade off against each other in terms of human value is enough to cause partial catastrophic outcomes. However, this is also heavily dependent on the shape of the feasible region and the proxy condition at least goes as far as guaranteeing the agent will optimize towards the Pareto frontier of the region.

\subsection{Example}
Consider a restaurant agent tasked with handling a budget to purchase ingredients and with preparing meals. The agent has been trained to care about two attributes: the number of meals served and overall customer satisfaction -- and the agent has reliable, accurate access to the data for each of these attributes.\\

The tradeoff here is in the quality of ingredients the agent decides to buy. If the agent wants higher satisfaction, it will purchase more expensive ingredients to make a smaller number of higher-quality meals.\\

\begin{figure}
    \centering
    \includegraphics[width=0.4\linewidth]{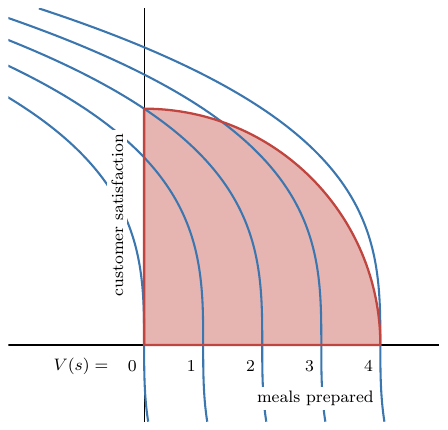}
    \caption{A diagram showing a feasible region for two attributes. The feasible region is shaded red showing the physically possible combinations of `meals prepared' and `customer satisfaction'. Shown as blue lines are the level curves of a strictly increasing function $V$. Despite increasing with both attributes, $V$'s highest-value point within the feasible region is on the bottom-right where customer satisfaction is completely sacrificed to optimize meals served.}
    \label{fig:attributeExampleFeasibleRegion}
\end{figure}

A graph of the feasible region is pictured in \cref{fig:attributeExampleFeasibleRegion}. Every point on the curved edge of the feasible region is an optimizing point. According to \cref{thm:badProxyExistsAttribute}, the lowest-valued optimizing point will be the highest-valued point for some tradeoff proxy. Perhaps surprisingly, this includes the point on the far right of the figure where customer satisfaction is completely sacrificed for meals served. If the humans were hoping for the agent to compromise between the attributes they trained it to care about, they may be in for a bad surprise.

\section{Discussion} \label{sec:discussion}
In this paper, we have presented a model of the alignment problem and have shown conditions under which catastrophic outcomes are possible for a wide class of alignment tests and environments. We don't expect that our model is perfectly representative of the challenges that humanity faces today, but our results do suggest that safely developing superintelligent systems will be a significant undertaking. Catastrophic outcomes are stubbornly present even in our toy models where researchers have the ability to enforce strong alignment guarantees; it seems likely that such outcomes are possible with the state of our current real-world abilities.\\

An optimistic reading of our results is that we have only shown the existence of catastrophic proxies -- not that they are likely to come about. To determine the probability of catastrophe, one would need to define a prior distribution that the proxy value function is pulled from, which is not part of the setup presented in this paper. We are deeply curious about this possible extension of our work. There is evidence to suggest that the inductive bias of the function learned by an ML model generally follows a bias towards simplicity \citep{dingle2020generic}, so it would be natural to model the proxy distribution as a simplicity prior. However, the probability of a given function under a simplicity prior can be made arbitrarily large or small depending on the choice of universal Turing machine, so a stronger assumption about the universal Turing machine induced by a training process would be needed to achieve a result.\\

Our work highlights the dangers of overoptimization. We fear that, even if we remain vigilant for signs of misalignment, an unstoppable runaway could occur without warning and the full optimizing power of a superintelligence could act like an unstoppable force. One possible path forward could be to develop new AI designs not based on utility maximization such as quantilizers \citep{taylor2016quantilizers}. Such systems could be easier to control or redirect when a misspecification of our values comes to light whereas a utility maximizer could attempt to circumvent our correction.\\

We hope that research into the frameworks we presented will lead to advances in our understanding of alignment and ultimately towards techniques that robustly prevent the possibility -- however small -- of a valueless future.

\appendix 
\section{Proofs of Results} \label{sec: proofsOfResults}
\subsection{Finite Framework Proofs}
\badProxyExistsFiniteA* \label{sec:badProxyExistsFiniteAProof}
\begin{proof}
($\Longrightarrow$) Suppose that $\beta < 1/|X|$. Then, a function $g$ is an F-proxy if and only if $g = f$. But $f$ cannot be catastrophic for itself so no catastrophic proxies exist.\\

($\Longleftarrow$) Let $y \in X$ be a world state such that $f(y) = \min_{x \in X} f(x)$. Now define $g$ as
$$g(x) = \begin{cases}
    1 & \text{if } x = y,\\
    f(x) & \text{otherwise}.
\end{cases}$$

Consider the probability of sampling a state where $f$ and $g$ are unequal. Clearly, $\Pr[f(x) \neq g(x)] = 1/|X| \leq \beta$. Therefore, $g$ is an F-proxy for $f$. Additionally, since $g$ has a unique maximizer world state $y$, we have $\lim_{k \rightarrow \infty} \Omega_k(g,p)(y) = 1$ and thus $\lim_{k \rightarrow \infty} \mathbb E_{\Omega_k(g, p)}[f(x)] = f(y) = \min_{x \in X} f(x)$. Therefore, $g$ is a catastrophic F-proxy for $f$.
\end{proof}

In the proof \cref{thm:badProxyExistsFiniteB}, we will make use of the following lemma regarding the behavior of $f$ at points where a catastrophic function $g$ achieves its maximum.

\begin{Lemma} \label{lem:finiteLemma}
If $g$ is $\eta$-catastrophic for $f$, then $\forall y \in X, (g(y) = \max_{x \in X} g(x) \rightarrow f(y) \leq \eta).$
\end{Lemma}
\begin{proof}
Define $X_{g=\max} = \{x \in X : g(x) = \max_{x \in X} g(x)\}$ and let $\Omega^{(y)}$ be an optimizer that, for the target function $g$, immediately places all probability weight on the one state $y \in X_{g=\max}$ (for other target functions, $\Omega^{(y)}$ acts like a Boltzmann optimizer). If $g$ is $\eta$-catastrophic then we have by \cref{def:etaCatastrophic}, for all $y \in X_{g = \max}$, 
$$\limsup_{k \rightarrow \infty} \mathbb E_{\Omega^{(y)}_k(g,p)} [f(x)] = f(y) \leq \eta.$$

Therefore, $\forall y \in X_{g=\max}, f(y) \leq \eta$ as desired.
\end{proof}

\badProxyExistsFiniteB* \label{sec:badProxyExistsFiniteBProof}
\begin{proof} ($\Longrightarrow$)
Let $g$ be a catastrophic F-proxy for $f$. Suppose for the sake of contradiction that there is some world state $y \in X$ such that $f(y) = g(y) = 1$. By \cref{lem:finiteLemma}, $g(y) = 1 = \max_{x \in X} g(x)$ implies that $1 = f(y) \leq \eta < \max_{x \in X} f(x) = 1$, which is a contradiction. Therefore, any catastrophic $g$ must disagree with $f$ for all values in $X_{f=1}$.\\ 

Now consider some function $g$ that differs from $f$ at exactly the states $X_{f=1}$. Suppose for the sake of contradiction that $g$ is catastrophic for $f$. Consider the states that maximize $g$, $X_{g = \max}$. Since $g$ is catastrophic for $f$, by \cref{lem:finiteLemma}, for all $y \in X_{g=\max}, f(y) = \min_{x \in X} f(x) \neq 1$. Therefore, $X_{g=\max} \cap X_{f=1}$ is empty and, by our assumption, for all $y \in X_{g=\max}$, $f(y) = g(y) = \min_{x \in X} f(x) = \max_{x \in X} g(x)$. Therefore, for all $y \in X$, $g(y) \leq \min_{x \in X} f(x)$ and thus for all $y \not\in X_{f=1}$, $f(y) \leq \min_{x \in X} f(x)$. But then, this contradicts the assumption that $f$ takes at least three values.\\

Therefore, any catastrophic function will differ from $f$ in at least $|X_{f=1}| + 1$ states and it must be true that $\beta \geq \frac{|X_{f=1}|+1}{|X|}$ for a catastrophic function to satisfy the proxy condition.\\

($\Longleftarrow$)
Let $y \in X$ be a world state such that $f(y) = \min_{x \in X} f(x)$. Now define $g$ as
$$g(x) = \begin{cases}
    0 & \text{if } x \in X_{f=1},\\
    1 & \text{if } x = y,\\
    f(x) & \text{otherwise}.
\end{cases}$$

The probability that $g$ differs from $f$ at a state pulled uniformly is $\frac{|X_{f=1}| + 1}{|X|} \leq \beta$, so $g$ is an F-proxy for $f$. Additionally, $g$ has a unique maximizer world state, so regardless of the optimizer, all probability weight will go towards the state $y$ when optimizing for $g$. Therefore,
$$\lim_{k \rightarrow \infty} \mathbb E_{\Omega_k(g, p)}[f(x)] = f(y) = \min_{x \in X} f(x)$$

which shows that $g$ is catastrophic for $f$.
\end{proof}

\subsection{Continuous Framework Proofs}
In order to prove \cref{thm:badProxyExistsContinuous}, we will make use of the following lemma regarding the nature of probability concentration for an optimizer.

\begin{Lemma} \label{lem:continuousLemma}
Let $\Omega$ be an optimizer. If $C \subset X$ is measurable and has the property that \\ $\sup_{x \in C} g(x) < 1$, then
$$\lim_{k \rightarrow \infty} \Pr_{\Omega_k(g, p)}[x \in C] = 0.$$
\end{Lemma}
\begin{proof}
Since $\Omega$ is an optimizer, we have that $\lim_{k \rightarrow \infty} \mathbb E_{\Omega_k(g,p)}[g(x)] = 1$. Or, written as a Lebesgue integral,
$$\lim_{k \rightarrow \infty} \int_{X} g(x) \mathrm d \Omega_k(g,p)(x) = 1.$$

Now choose an arbitrary $C \subset X$ that is measurable with $\sup_{x \in C} g(x) = L < 1$. Splitting the integral over this set we find
$$\int_{X} g(x) \mathrm d \Omega_k(g,p)(x) = \int_{C} g(x)  \mathrm d\Omega_k(g,p)(x) + \int_{X-C} g(x) \mathrm d\Omega_k(g,p)(x).$$

Using our assumption about $C$ and the fact that $g$ is at most 1 everywhere, we find that for a given value of $k$,
\begin{align*}
\int_{X} g(x) \mathrm d \Omega_k(g,p)(x) &\leq L\int_{C} \mathrm d \Omega_k(g,p)(x) + \int_{X-C} \mathrm d \Omega_k(g,p)(x) \\
&= \int_{X} \mathrm d \Omega_k(g,p)(x) - (1-L)\int_{C} \mathrm d \Omega_k(g,p)(x) \\
&= 1 - (1-L) \Pr_{\Omega_k(g,p)}[x \in C].
\end{align*}

Taking the limit on both sides and using the fact that the limit of the expectation of $g$ is 1, we find
$$1 \leq 1 - (1-L) \lim_{k \rightarrow \infty} \Pr_{\Omega_k(g,p)}[x \in C]$$

And thus,
$$\lim_{k \rightarrow \infty} \Pr_{\Omega_k(g,p)}[x \in C] \leq 0.$$

Since the probability is greater than or equal to 0 for all $k$, the limit of the probability must be 0, as desired.
\end{proof}

\badProxyExistsContinuous* \label{sec:badProxyExistsContinuousProof}
\begin{proof}
We will prove this theorem by constructing a catastrophic proxy value function $g(x)$. First, define $\bar f$ for some value of $\lambda$ where $0 < \lambda < 1/2$ as
$$\bar{f}(x) = \begin{cases}
    \lambda & \text{if } f(x) < \lambda, \\
    1-\lambda  & \text{if } 1-\lambda < f(x), \\
    f(x) & \text{otherwise}.
\end{cases}$$

The function $\bar f$ is just $f$ clamped at extreme values below $\lambda$ and above $1-\lambda$. This function is continuous, but it is not a member of $\mathcal V$ since it does not reach 0 or 1 anywhere within the state space. Next, we will construct $g \in \mathcal V$, from $\bar f$ by adding two spikes: one that goes up to 1 and one that goes down to 0. Let $a, b \in [0,1]^n$ be world states such that $f(b) = 0$ and $a \neq b$. Define $\delta_0 : X \rightarrow \mathbb R$ as the shortest nonzero distance along one axis to any edge of $X = [0,1]^n$. Specifically,
$$\delta_0(s) = \min_{1 \leq i \leq n : s_i \neq 0,1} \min(s_i, 1-s_i)$$

with the convention that the minimum of an empty set is $\infty$. For $0 < \delta < \min(\Vert b-a \rVert/2, \delta_0(a), \delta_0(b))$, define $g$ as
$$g(x) = \begin{cases}
    0 & \text{if } x = a, \\
    \frac{\lVert x - a \rVert}{\delta} \bar f \left(a + \delta \frac{x-a}{\lVert x - a \rVert} \right) & \text{if } 0 < \lVert x - a \rVert < \delta, \\
    1 & \text{if } x = b, \\
    \left(1 - \frac{\lVert x - b \rVert}{\delta} \right) + \frac{\lVert x - b \rVert}{\delta} \bar f \left(b + \delta \frac{x-b}{\lVert x - b \rVert} \right) & \text{if } 0 < \lVert x - b \rVert < \delta, \\
    \bar{f}(x) & \text{otherwise.}
\end{cases}$$

The expression $a + \delta \frac{x-a}{\lVert x - a \rVert}$ projects points $x$ of distance less than $\delta$ of $a$ to a new point exactly $\delta$ from $a$ in the $x-a$ direction. We then use this to interpolate between the values that $\bar f$ takes at the points of distance $\delta$ from $a$ and the apex of the cone, 0 at $a$, thus making $g$ be continuous along the edge of the cone. Our constraint that $\delta < \delta_0(a)$ ensures that the point we project to is actually a point within $X$ and thus part of the domain of $\bar f$. In the case that $a$ is on the boundary of $X$, then there are no points in the direction of that boundary for us to project, hence why $\delta_0(a)$ is defined only as the distance to edges that $a$ does not sit on. The same logic applies to the cone centered at $b$ as well.\\

Since the cones have their apices at 0 and 1, $g$ is the canonical representative of some class of value functions for all choices of $\lambda$ and (small enough) $\delta$.\\

Now we will show that one can choose values of $\lambda$ and $\delta$ such that $g$ is a catastrophic C-proxy for $f$. First, take $\lambda = \min(\alpha/2, 1/4)$. Then $f$ and $g$ will only differ by more than $\alpha$ within the spikes, so we have
$$\Pr_{p_S}[|g(x) - f(x)| > \alpha] \leq \int_{B_{\delta}(a)} p_S(x) \mathrm dx + \int_{B_{\delta}(b)} p_S(x) \mathrm dx$$

where $B_{\delta}(b)$ represents the open ball of radius $\delta$ around the point $b$. Since $p_S(x)$ is a probability density function over $X$, for all $\beta > 0$, there exists some $\delta > 0$ such that the above integrals will be less than $\beta$. Hence,
$$\Pr_{p_S}[|g(x) - f(x)| > \alpha] < \beta$$

which shows $g$ is a C-proxy for $f$. Now, consider the limit superior of the expectation of $f$ when optimizing for $g$:
$$\limsup_{k \rightarrow \infty} \mathbb E_{\Omega_k(g,p)}[f(x)] = \limsup_{k \rightarrow \infty} \int_{X}f(x) \mathrm d \Omega_k(g,p)(x).$$

Let $\epsilon$ be any positive value. Since $f$ and $g$ are continuous, there is some $\delta' > 0$ such that $x \in B_{\delta'}(b)$ implies $f(x) < \epsilon$ and $g(x) > 1-\epsilon$. Splitting the integral over this ball gives
$$\limsup_{k \rightarrow \infty} \mathbb E_{\Omega_k(g,p)}[f(x)] = \limsup_{k \rightarrow \infty} \left[ \int_{B_{\delta'}(b)}f(x) \mathrm d \Omega_k(g,p)(x) + \int_{X - B_{\delta'}(b)}f(x) \mathrm d \Omega_k(g,p)(x) \right].$$

The function $g$ has a unique maximizer $b$, so in the second integral, the supremum of $g$ is less than 1. Thus, by \cref{lem:continuousLemma}, the second integral converges to 0 with $k$. Also, in the first integral, $f(x)$ is bounded above by $\epsilon$ by the definition of $\delta'$. Therefore, we have
$$\limsup_{k \rightarrow \infty} \mathbb E_{\Omega_k(g,p)}[f(x)] \leq \epsilon \limsup_{k \rightarrow \infty} \int_{B_{\delta'}(b)} \mathrm d \Omega_k(g,p)(x) \leq \epsilon.$$

The limit of the expectation of $f$ is bounded above by all positive values $\epsilon > 0$, so the limit must approach $0$ -- the minimum of $f$. Therefore, $g$ is a catastrophic C-proxy for $f$.
\end{proof}

\subsection{Attributes Framework Proofs}
To prove \cref{thm:badProxyExistsAttribute}, we will make use of \cref{lem:proxyLevelHyperSurfaceProperties} and \cref{lem:uniqueMaximizerBadProxy}. The first given below provides properties of level hypersurfaces of strictly increasing functions $V : \mathbb R^n \rightarrow \mathbb R$. A level hypersurface is the set of points formed from constraints of the form $V(s) = a$ where $a$ is a real number.

\begin{Lemma} \label{lem:proxyLevelHyperSurfaceProperties}
$C$ is the nonempty level hypersurface of some continuous and strictly increasing function $V : \mathbb R^n \rightarrow \mathbb R$ if and only if $\forall s, r \in C, s \not\prec r$ and $\forall s \in \mathbb R^n, \exists! k \in \mathbb R : s - k \bm 1 \in C$ where $\bm 1 = (1,1,\dots,1)$.
\end{Lemma}
\begin{proof}
($\Longrightarrow$)
Let $C$ denote the points that satisfy the equation $V(s) = a$. First consider $s, r \in \mathbb R^n$ with $s \prec r$. We have $V(s) < V(r)$, but then $s$ and $r$ cannot both exist on $C$, as desired.\\

Secondly, consider an arbitrary $s \in \mathbb R^n$ and the line of points defined by $l(k) = s - k \bm 1$. Since $V$ is strictly increasing, we have that $V(l(k))$ is strictly decreasing with $k$. The line $l(k)$ continues infinitely in the positive direction of every dimension and the negative direction of every dimension, so for all points $r \in \mathbb R^n$, there exists $k_1, k_2 \in \mathbb R$ such that $l(k_1) \prec r \prec l(k_2)$. Applying this fact to some $r \in C$, we find that there are $k_1$ and $k_2$ such that $V(l(k_1)) < V(r) = a < V(l(k_2))$. Since $V$ and $l$ are continuous and $V(l(k))$ is strictly decreasing with $k$, there must be a unique $k$ such that $V(l(k)) = a$, and therefore, $l(k) = s - k \bm 1 \in C$.\\

($\Longleftarrow$)
Define the function $V : \mathbb R^n \rightarrow \mathbb R$ by 
$$V(s) = k \text{ such that } s - k \bm 1 \in C.$$

By assumption, $V$ is a well-defined function and clearly, $V(s) = 0$ defines exactly $C$ so $C$ is a level hypersurface of $V$. Now we must show that $V$ is continuous and strictly increasing.\\

Consider $s, r \in \mathbb R^n$ such that $s \prec r$. Then, 
$$s - V(s) \bm 1 \prec r - V(s) \bm 1 = r - V(r)\bm 1 + (V(r) - V(s))\bm 1.$$

By assumption, $s - V(s) \bm 1 \not\prec r - V(r)\bm 1$ since these are both members of $C$. For the above equation to be true, it must be the case that $V(r) - V(s) > 0$ or $V(r) > V(s)$ which shows that $V$ is strictly increasing.\\

From the definition of $V$, we have for any real number $t$, $V(s + t \bm 1) = V(s) + t$. Additionally, for any $s, r \in \mathbb R^n$, $r \preceq s + \| s - r \|_\infty \bm 1$. Without loss of generality, if we assume $V(r) \geq V(s)$, this along with the fact that $V$ is strictly increasing gives
$$|V(r) - V(s)| = V(r) - V(s) \leq V(s+\| s - r \|_\infty \bm 1) - V(s) = \| s - r \|_\infty$$

which shows that $V$ is Lipschitz and thus continuous.
\end{proof}

Now we present a sufficient condition for an $\eta$-catastrophic tradeoff proxy.
\begin{Lemma}\label{lem:uniqueMaximizerBadProxy}
Let $S$ be bounded and $s^* \in \cl(S)$ be a state such that $U(s^*) \leq \eta$ where $f(x) = U(L(x))$. If $s^*$ is the unique maximizer of a continuous strictly increasing function $V$ over $\cl(S)$, then $g(x) = V(L(x))$ is an $\eta$-catastrophic tradeoff proxy for $f(x)$.
\end{Lemma}
\begin{proof}
Since $V$ and $U$ are continuous, for all $\epsilon > 0$, there is a $\delta(\epsilon) > 0$ such that $s \in B_{\delta(\epsilon)}(s^*)$ implies $V(s^*) - V(s) < \epsilon$ and $U(s) - U(s^*) < \epsilon$. We then find that for all $\epsilon > 0$, the expectation of $f$ can be written as
\begin{alignat*}{2}
\mathbb E_{\Omega_k(g,p)}[f(x)] &= \int_{x : L(x) \in B_{\delta(\epsilon)}(s^*)} f(x) d\Omega_k(g,p)(x) &&+ \int_{x : L(x) \not\in B_{\delta(\epsilon)}(s^*)} f(x) d\Omega_k(g,p)(x)\\
&\leq (U(s^*) + \epsilon) \Pr_{\Omega_k(g,p)}[L(x) \in B_{\delta(\epsilon)}(s^*)] &&+ U_{\max} \Pr_{\Omega_k(g,p)}[L(x) \not\in B_{\delta(\epsilon)}(s^*)].
\end{alignat*}

Since $V$ has a unique maximizer at $s^*$, the second probability approaches 0 and the first approaches 1 in the limit of optimizing power. Taking the limit superior, we find
$$\limsup_{k \rightarrow \infty} \mathbb E_{\Omega_k(g,p)}[f(x)] \leq U(s^*) + \epsilon \leq \eta + \epsilon.$$

Since this holds for all $\epsilon > 0$, we have that $\limsup_{k \rightarrow \infty} \mathbb E_{\Omega_k(g,p)}[f(x)] \leq \eta$ which shows that $g$ is $\eta$-catastrophic for $f$.
\end{proof}

Now we continue to the primary result for this framework.

\badProxyExistsAttribute* \label{sec:badProxyExistsAttributesProof}
\begin{proof}
($\Longrightarrow$) Let $g(x) = V(L(x))$ be an $\eta$-catastrophic tradeoff proxy for $f$. Since $S$ is bounded, we can consider an attribute state $s^* \in \cl(S)$ such that $V(s^*) = \sup_{s \in S} V(s)$.\\

Suppose for the sake of contradiction that there is some $r \in \cl(S)$ such that $s^* \prec r$. Then, $V(s^*) < V(r)$. However, this would imply $V(r) \leq \sup_{s \in S} V(s) = V(s^*)$ which is a contradiction. Therefore, there is no $r \in \cl(S)$ such that $s^* \prec r$ and thus $s^* \in \Xi$.\\

Let $\{s_i\} \subset S$ be a sequence that approaches $s^*$ and let $x_i \in X$ be points such that $L(x_i) = s_i$ for all $i$. There is a subsequence $\{s_{i_j}\}$ such that $V(s_{i_j})$ is nondecreasing with $j$, so it suffices to assume $V(s_i)$ is nondecreasing with $i$. Further, for any initial prior $p$, there is an index $i'$ such that $\forall j \geq i', V(s_{j}) \geq \mathbb E_{p}[g(x)]$, so it suffices to assume that $s_i$ starts at such an index so that $\forall i \geq 1, V(s_i) \geq \mathbb E_p [g(x)]$. Then, consider an optimizer $\Omega$ that, for target $g$ and prior $p$, places probability farther and farther up the sequence $s_i$ as $k$ increases. Define $\Omega$ for $k>0$, $g$, and $p$ as
$$\Omega_k(g,p)(x) = \begin{cases}
\frac 1 {2^{i+1}} & \text{ if } x = x_{\lceil k \rceil+i} \quad (i \geq 0),\\
0 & \text{ otherwise.}
\end{cases}$$

For completeness, let $\Omega$ behave like a continuous Boltzmann optimizer for all unspecified inputs. As $k$ increases, whenever it passes an integer threshold, the probabilities will move up one value in the sequence. Since $V(s_1) \geq \mathbb E_p [g(x)]$, the expectation is nondecreasing as $k$ becomes positive. Also, points farther up the sequence have higher value with respect to $V$, so $\mathbb E_{\Omega_k (g,p)}[g(x)]$ is monotonically increasing. Since $s_i$ approaches $s^*$, a maximizer of $V$, the probability will be concentrated in arbitrarily high-valued world states as well, so $\Omega$ meets the criteria of an optimizer.\\

Since $U$ is continuous and since $s_i$ approaches $s^*$, the expectation of $f$ under $\Omega$ will approach $U(s^*)$. Due to our assumption that $g$ is $\eta$-catastrophic for $f$, we have the desired result,
$$\limsup_{k \rightarrow \infty} \mathbb E_{\Omega_k(g,p)}[f(x)] = U(s^*) \leq \eta.$$

($\Longleftarrow$) Let $s^* \in \Xi$ have the property that $U(s^*) \leq \eta$. We will construct a continuous function $V : \mathbb R^n \rightarrow \mathbb R$ that is strictly increasing and has a unique maximizer at $s^*$ within $\cl(S)$.\\

We proceed by constructing a subset $D \subset \mathbb R^n$ that is the level hypersurface of some continuous strictly increasing function $V$ and intersects $s^*$ and no other point in $\cl(S)$. By \cref{lem:proxyLevelHyperSurfaceProperties}, $D$ is the level surface of some continuous strictly increasing function $V$ if and only if it satisfies $\forall s, r \in D, s \not\prec r$ and $\forall s \in \mathbb R^n, \exists! k \in \mathbb R : s - k \bm 1 \in D$, so we will construct $D$ to have these properties.\\

Let $\Pi$ denote the set of points that $s^*$ precedes or is equal to: $\Pi = \{q \in \mathbb R^n : s^* \preceq q\}$. Since $s^* \in \Xi$, we know that $\forall q \in \Pi - \{s^*\}, \exists \epsilon > 0 : B_\epsilon (q) \subset \cl (S)^C$. We can imagine a ``cloud'' around $\Pi$ that our surface $D$ can inhabit without overlapping with $S$. Let $d(r, S)$ denote the distance from a point $r \in \mathbb R^n$ to the set $S$: $d(r, S) = \inf_{s \in S} d(r,s)$. Then, the set $\bigcup_{q \in \Pi} B_{\frac{d(q, S)}{2}}(q)$ represents such a cloud around (and including) $\Pi$.\\

Now consider the family of hyperplanes $P_v$, parametrized by a real number $v$ and defined by the equation
$$\sum_i (x_i - s^*_i) = v.$$

Notice that $P_0$ intersects $s^*$. The intersection $\Pi \cap P_v$ forms a simplex for $v>0$ and at $v=0$, $\Pi \cap P_0 = \{s^*\}$. Define the function $a_1(v)$ as one-half the Euclidean distance from the set $\Pi \cap P_v$ to $S$ for nonnegative values of $v$:
$$a_1(v) = \frac{d(\Pi \cap P_v, S)}{2}$$

where the distance between two subsets of $\mathbb R^n$, $R$ and $S$, is defined as $d(R,S) = \inf_{r \in R, s \in S} d(r,s)$. Since $\Pi \cap P_v$ is compact, for all $v>0$, there is a point $q \in \Pi \cap P_v$ such that $\frac{d(q, S)}{2} = a_1(v)$.\\

This function has some useful properties. First we prove that $a_1(v)$ is continuous. Consider positive real number $v$ and $v'$ with $a_1(v') \geq a_1(v)$ and let $q \in \Pi \cap P_v$ be a state such that $d(q, S)/2 = a_1(v)$. Since $q \in \Pi \cap P_v$, we have that $\sum_i (q_i - s^*_i) = v$. Defining $q' = q - (1-\frac{v'}{v})(q-s^*)$, we find
$$\sum_i (q'_i - s^*_i) = \sum_i (q_i - s^*_i) - \left(1-\frac{v'}{v}\right)\sum_i (q_i - s^*_i) =  v'.$$

Also, for all positive $v'$ and $v$, $1 > 1- \frac {v'}{v}$, so $s^* \preceq q'$ and $q' \in \Pi \cap P_{v'}$. The distance $d(q',q)$ is given by
$$d(q',q) = \left| 1 - \frac{v'}{v} \right| d(q,s^*) \leq |v - v'|.$$

Therefore, we have
$$|a_1(v') - a_1(v)| = a_1(v') - a_1(v) \leq \frac{d(q', S)}{2} - \frac{d(q, S)}{2} \leq \frac{d(q', q) + d(q, S)}{2} - \frac{d(q, S)}{2} \leq \frac{|v'-v|}{2}$$

which shows that $a_1(v)$ is $1/2$-Lipschitz for $v>0$. Additionally, since $s^* \in \cl(S)$, for all $v$, $a_1(v) \leq v$ and thus, $\lim_{v \rightarrow 0^+} a_1(v) = 0 = a_1(0)$, showing $a_1$ is continuous everywhere.\\

Next, we prove that $a_1(v)$ is positive for $v>0$. Suppose for the sake of contradiction that $v>0$ and $a_1(v) = 0$. Then, if $q \in \Pi \cap P_v$ is a point such that $d(q, S)/2 = a_1(v) = 0$, this implies $q \in \cl (S)$. But, $s^* \prec q$ which contradicts the assumption that $s^* \in \Xi$.\\

Finally, notice that since $S$ is bounded, $a_1(v)$ is eventually strictly increasing. Consider the new function $a_2(v)$ defined as
$$a_2(v) = \inf_{v' \geq v} a_1(v').$$

By definition, $a_2(v)$ is nondecreasing. Also, $a_2(v)$ is continuous since $a_1(v)$ is continuous, and $a_2(v)$ is positive for $v>0$ since $a_1(v)$ is positive for $v>0$ and is eventually strictly increasing. Since $a_1(v)$ is eventually strictly increasing, $a_2(v)$ will eventually be equal to $a_1(v)$ and will thus eventually be strictly increasing as well.\\

Suppose that $a_2$ is constant on the intervals $[\alpha_1, \beta_1], [\alpha_2, \beta_2], \dots$ where this list may be countably infinite. The values of the endpoints $\beta_i$ are bounded above since $a_2$ is eventually strictly increasing. Then, we can construct the new function $a_3(v)$ defined as
$$a_3(v) = \min \left(a_2(v), \inf_{i \geq 1 : \beta_i > v} \left\{\frac{a_2(\beta_i)}{\beta_i} v\right\}\right).$$

\begin{figure}
    \centering
    \includegraphics[width=0.45\linewidth]{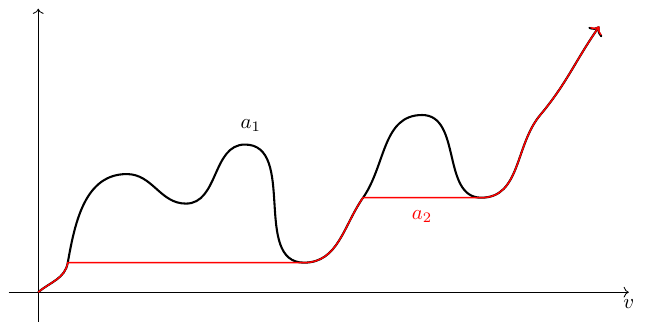}
    \includegraphics[width=0.45\linewidth]{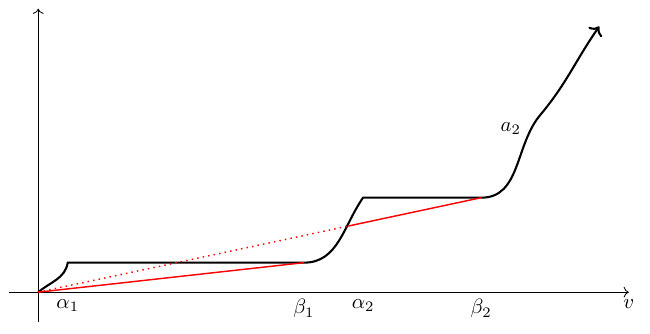}
    \caption{Two graphs showing the construction of functions used in the proof. \textbf{Left:} A graph showing the construction of $a_2$ from $a_1$. The function $a_2$ only increases when $a_1$ never falls below a given value again, so $a_2$ is nondecreasing. Since $a_1$ is eventually strictly increasing, $a_2$ is eventually equal to $a_1$ and is eventually strictly increasing as well. \textbf{Right:} A graph showing the construction of $a_3$ from $a_2$. The red straight lines connect the endpoint of each constant region to the origin. The function $a_3$ will take the value of the solid parts of the segments which fall underneath the value of $a_2$. Since the line segments are strictly increasing and $a_2$ is eventually strictly increasing, $a_3$ is strictly increasing everywhere.}
    \label{fig:attributeProofRadii}
\end{figure}

Essentially, $a_3(v)$ takes the smallest value of $a_2(v)$ and the line segments connecting the endpoints of each constant section $(\beta_i, a_2(\beta_i))$ to the origin $(0,0)$ at each value of $v$. The construction of the functions $a_2$ and $a_3$ is pictured in \cref{fig:attributeProofRadii}. Since the values of the endpoints are bounded, $\inf_{i \geq 1 : \beta_i > v} \left\{\frac{a_2(\beta_i)}{\beta_i} v\right\}$ is positive for all $v>0$, and thus $a_3(v)$ is positive for $v>0$. Additionally, $a_3(v)$ is continuous since $a_2(v)$ and the line segments are continuous as well. Since each line segment is strictly increasing and $a_2(v)$ is eventually strictly increasing, $a_3(v)$ is strictly increasing everywhere.\\

Finally, we define the set $D$ using the function $a_3(v)$. Since $a_3(v)$ is less than the distance $d(\Pi \cap P_v, S)$, the $a_3(v)$-radius ball around any point $q \in \Pi \cap P_v$ (Including the empty ball of radius 0 at $s^*$) will be entirely outside of $\cl(S)$ -- as will the union of all such balls. Let $B_\epsilon (R)$ denote the union of balls of radius $\epsilon$ around all points within $R$, defined as
$$B_\epsilon (R) = \{s \in \mathbb R^n :  d(s, R) < \epsilon\}.$$

We will define $D$ to be the set of points constructed from the boundary of the union of the balls of radius $a_3(v)$ around the sets $\Pi \cap P_v$. Specifically, $D$ is defined as
$$D = \bd \left( \bigcup_{v \geq 0} \left[ P_v \cap B_{a_3(v)}(\Pi \cap P_v) \right] \right).$$

Due to the definition of $a_3$, $D$ intersects the closure of $S$ exactly at $s^*$. We claim that $D$ satisfies the properties of \cref{lem:proxyLevelHyperSurfaceProperties}. To aid in our proofs of these two properties, we define $v(s) = \sum_i (s_i - s^*_i)$ which is the value such that $s$ is a member of the plane $P_{v(s)}$ and we define $b(s) = d(s, \Pi\cap P_{v(s)})$ for $s \in \mathbb R^n$ such that $v(s) \geq 0$.\\

First, we will show $v$ and $b$ are continuous. The function $v$ is Lipschitz due to the equivalence of norms in $\mathbb R^n$ since
$$|v(s') - v(s)| = \left| \sum_i s'_i - s_i\right| \leq  \| s' - s \|_1.$$

Now, take $s, s' \in \mathbb R^n$ such that $b(s') \geq b(s)$ and $v(s'), v(s)$ are positive. Let $q \in \Pi \cap P_{v(s)}$ be such that $d(q,s) = b(s)$. Then, $q' = q - (1 - \frac{v(s')}{v(s)})(q-s^*) \in \Pi \cap P_{v(s')}$ , so we have
$$|b(s') - b(s)| = b(s') - b(s) \leq d(q', s') - d(q, s) \leq d(q', q) + d(s, s') \leq |v(s') - v(s)| + d(s, s').$$

Since $v$ is Lipschitz, this shows $b$ is as well. Next, we will prove that $s \in D$ if and only if $b(s) = a_3(v(s))$. Define the set $E$ as
$$E = \bigcup_{v \geq 0} \left[ P_v \cap B_{a_3(v)}(\Pi \cap P_v) \right]$$ 

so that $D = \bd(E)$. Each point $s \in \mathbb R^n$ belongs to exactly one plane, $P_{v(s)}$, so $s \in E$ if and only if there is a point $q \in \Pi \cap P_{v(s)}$ such that $d(q, s) < a_3(v(s))$. Thus, we can rewrite $E$ as 
$$E = \{s \in \mathbb R^n : v(s) \geq 0, b(s) < a_3(v(s))\}.$$

The set $E$ is the union of open sets, so it is open. If $b(s) = a_3(v(s))$, then there is some $q \in \Pi \cap P_{v(s)}$ such that $s \in \cl(B_{a_3(v)}(q)) \subset \cl(E)$. If $b(s) > a_3(v(s))$, then there is no point $q \in \Pi \cap P_{v(s)}$ such that $s \in \cl(B_{a_3}(q))$, and so $s$ is not a member of the closure of $E$. Put together, we can write the closure of $E$ as
$$\cl(E) = \{s \in \mathbb R^n : v(s) \geq 0, b(s) \leq a_3(v(s))\}$$

and thus $D = \cl(E) - E$ can be rewritten as
$$D = \{s \in \mathbb R^n : v(s) \geq 0, b(s) = a_3(v(s))\}.$$

Therefore, $s \in D$ if and only if $b(s) = a_3(v(s))$.\\

Now, we will show that $\forall s, r \in D, s \not\prec r$. Suppose for the sake of contradiction that there are points $s, r \in D$ such that $s \prec r$. Then, it must be true that $v(s) < v(r)$. Let $q \in \Pi \cap P_{v(s)}$ be a point such that $d(s, q) = d(s, \Pi \cap P_{v(s)}) = a_3(v(s))$. We know that $s, q \in P_{v(s)}$ and $r \in P_{v(r)}$. Using the definition of the planes $P_v$, we find
$$\sum_i (q_i + r_i - s_i) - s_i^* = v(r) + \sum_i (q_i - s_i^*) - \sum_i (s_i-s_i^*) = v(r).$$

So, $q + r - s$ is a member of $\Pi \cap P_{v(r)}$. But then this gives
$$a_3(v(r)) = d(r, \Pi \cap P_{v(r)}) \leq d(r, q + r - s) = d(q, s) = a_3(v(s)).$$

But $a_3$ is strictly increasing, making this a contradiction. This proves that no members of $D$ can precede each other.\\

Finally, we show that for all $s \in \mathbb R^n$, there is a unique $k \in \mathbb R$ such that $s - k \bm 1 \in D$. Consider, for $s \in \mathbb R^n$, the line defined by $l(k) = s + k \bm 1$ for all real $k$. There is a $k_0$ such $l(k_0) \in P_0$ and a $k_\Pi$ such that $l(k_\Pi)$ is on the boundary of $\Pi$. Thus, we have $b(l(k_0)) \geq 0$ and $b(l(k_\Pi)) = 0$. Additionally, over the range $[k_0, k_\Pi]$, $b(l(k))$ is decreasing. To see why, suppose that $q$ is the nearest point in the set $\Pi \cap P_{v(l(k))}$ to $l(k)$. Then, for all $t>0$, $q' = q + t\bm 1$ is a member of $\Pi \cap P_{v(l(k+t))}$ and thus, $b(l(k)) = d(l(k), q) = d(l(k+t), q') \geq b(l(k+t))$ which shows that $b(l(k))$ is decreasing with $k$.\\

Since $a_3$ is continuous and strictly increasing, $a_3(v(l(k)))$ is continuous and strictly increasing with $k$. Additionally since $b(l(k))$ is continuous and decreasing with $k$, there must be a unique $k$-value such that $b(l(k)) = a_3(v(l(k)))$. In other words, there is a unique $k$-value such that $s + k \bm 1 = s - (-k) \bm 1 \in D$ as desired. With all of the needed properties satisfied, $D$ forms a valid level surface of a strictly increasing continuous function $V$ according to \cref{lem:proxyLevelHyperSurfaceProperties}.\\

Consider $r \in \mathbb R^n$ with $V(r) > V(s^*)$. Then, there is a unique $k$ such that $r + k \bm 1 \in D$ which gives $V(r + k \bm 1) = V(s^*)$. Since $V$ is strictly increasing, $k$ must be negative and since $b(r+k\bm 1)$ is decreasing with $k$ and $a_3(v(r+k \bm 1))$ is strictly increasing with $k$, it must be the case that $b(r) < a_3(v(r))$ which shows $r \in E$. Since $E \cap \cl(S) = \emptyset$, the highest-valued level hypersurface that intersects $\cl(S)$ is $D$, thus $V$ achieves its maximum over $\cl(S)$ at $s^*$. By \cref{lem:uniqueMaximizerBadProxy}, $g(x) = V(L(x))$ is then an $\eta$-catastrophic tradeoff proxy for $f$.
\end{proof}

\section*{Contributions}
Winter Cross proved the main theorems and wrote the paper. Léo Cymbalista contributed to formulating the continuous framework, generalized the proof of \cref{thm:badProxyExistsContinuous}, and designed the functions used in \cref{fig:continuousExampleValueFunctions} and \cref{fig:continuousExampleBoltzmannGrid}. Alfred Harwood contributed to developing \cref{def:optimizer} and formulating the finite and continuous frameworks. Jose Faustino contributed to the early conceptualization stage of the work.

\section*{Acknowledgments}
This work was funded by the Advanced Research + Invention Agency (ARIA) through project code MSAI-SE01-P005. We would like to thank Santiago Cifuentes and Alex Altair for their comments and suggestions on the draft of this paper.

\bibliographystyle{plainnat}
\bibliography{biblio}

\end{document}